\documentclass[envcountsame,envcountsect]{llncs}
\usepackage{scrextend}
\usepackage{times}
\usepackage{helvet}
\usepackage{courier}
\usepackage{amsmath}
\usepackage{amssymb}
\usepackage{upgreek}
\usepackage{url}
\usepackage{hyperref}
\usepackage[ruled,linesnumbered]{algorithm2e}
\usepackage[misc,geometry]{ifsym} 
\usepackage{booktabs}
\usepackage{multirow} 
\usepackage{caption}
\usepackage{subcaption}
\usepackage{tikz,ifthen,pgfplots}
\usetikzlibrary{arrows,trees,backgrounds,automata,shapes,decorations,plotmarks,fit,calc,positioning,shadows,chains}
\tikzstyle{every pin edge}=[<-,shorten <=1pt]
\tikzstyle{neuron}=[circle,draw=black, thick, minimum size=17pt,inner sep=0pt]
\tikzstyle{input neuron}=[neuron]
\tikzstyle{output neuron}=[neuron]
\tikzstyle{hidden neuron}=[neuron]
\tikzstyle{annot} = [text width=4em, text centered]
\usetikzlibrary{calc}
\usetikzlibrary{angles,quotes}
\usepgfplotslibrary{groupplots}

\newcommand{\xiaowei}[1]{{\color{blue}#1}}

\newcommand{\sat}{\texttt{SAT}}
\newcommand{\unsat}{\texttt{UNSAT}}

\newcommand{\cstVarSet}{\mathbf{C}}
\newcommand{\symVarSet}{\mathbf{S}}
\newcommand{\absEle}{\mathbf{n}^{\#}}
\newcommand{\absDom}{\mathbf{N}^{\#}}
\newcommand{\symM}{\xi}
\newcommand{\absSem}[1]{[\kern-0.15em[#1]\kern-0.15em]^{\sharp}}

\newcommand{\algJoin}{\mathbf{join}}
\newcommand{\mlexpr}{\mathbf{expr}}
\newcommand{\air}{AI\(^2\)-r}

\title{Analyzing Deep Neural Networks with Symbolic Propagation: Towards Higher Precision and Faster Verification}
\author{
    Jianlin Li\inst{1,2} 
    \and 
    Jiangchao Liu\inst{3} 
    \and 
    Pengfei Yang\inst{1,2} 
    \and \\
    Liqian Chen(\Letter)\inst{3} 
    \and 
    Xiaowei Huang\inst{4,5} 
    \and 
    Lijun Zhang\inst{1,2,5}
}
\institute{
State Key Laboratory of Computer Science, Institute of Software, Chinese Academy of Sciences, Beijing, China
\and
University of Chinese Academy of Sciences, Beijing, China
\and
National University of Defense Technology, Changsha, China
\and
University of Liverpool, Liverpool, UK
\and
Institute of Intelligent Software, Guangzhou, China
\\ \email{lqchen@nudt.edu.cn}
}

\newtheorem{propn}{Proposition}

\newcommand{\commentout}[1]{}

\begin{document}

\maketitle
\begin{abstract}
Deep neural networks (DNNs) have been shown lack of robustness, as they are vulnerable to small perturbations on the inputs,
which has led to safety concerns on applying DNNs to safety-critical domains.
Several verification approaches have been developed to automatically prove or disprove  safety properties for DNNs.
However, these approaches suffer from either the scalability problem, i.e., only small DNNs can be handled, or the precision problem, i.e., the obtained bounds are loose.
This paper improves on a recent proposal of analyzing DNNs through the classic abstract interpretation technique, by a novel symbolic propagation technique.
More specifically, the activation values of neurons are represented \emph{symbolically} and propagated forwardly from the input layer to the output layer, on top of abstract domains.
We show that our approach can achieve significantly higher precision and thus can prove more properties than using only abstract domains.
Moreover, we show that the bounds derived from our approach on the hidden neurons, when applied to a state-of-the-art SMT based verification tool, can improve its performance.
We implement our approach into a software tool and validate it over a few DNNs trained on benchmark datasets such as MNIST, etc.
\end{abstract}

\section{Introduction}

During the last few years, deep neural networks (DNNs) have been broadly applied in various domains including nature language processing \cite{DBLP:journals/spm/X12a}, image classification \cite{DBLP:conf/nips/KrizhevskySH12}, game playing \cite{alphago}, etc.
The performance of these DNNs, when measured with the prediction precision over a test dataset, is comparable to, or even better than, that of manually crafted software.
However, for safety-critical applications, it is required that the DNNs are certified against properties related to its safety.
Unfortunately, DNNs have been found lack of robustness. Specifically, \cite{SZSBEGF2014} discovers that it is possible to add a small, or even impcerceptible, perturbation to a correctly classified input and make it misclassified. Such adversarial examples have raised serious concerns on the safety of DNNs. Consider a self-driving system controlled by a DNN.
A failure on the recognization of a traffic light may lead to a serious consequence because human lives are at stake.

Algorithms used to find adversarial examples are based on either gradient descent, see e.g., \cite{SZSBEGF2014,CW-Attacks}, or saliency maps, see e.g., \cite{JSMA}, or evolutionary algorithm, see e.g., \cite{AJJ2014}, etc. Roughly speaking, these are heuristic search algorithms without the guarantees to find the optimal values, i.e., the bound on the gap between an obtained value and its ground truth is unknown. However, the certification of a DNN needs 
provable guarantees.
Thus, techniques based on formal verification have been developed. Up to now, DNN verification includes constraint-solving \cite{DBLP:conf/cav/PulinaT10,DBLP:conf/cav/KatzBDJK17,LM2017,DBLP:conf/atva/Ehlers17,NKPSW2017,DBLP:conf/icml/WongK18,DBLP:journals/corr/abs-1803-06567}, layer-by-layer exhaustive search \cite{DBLP:conf/cav/HuangKWW17,WHK2017,DBLP:conf/icml/WengZCSHDBD18}, global optimization \cite{DBLP:conf/ijcai/RuanHK18}, abstract interpretation \cite{AI2,deeppoly,deepz}, etc. 
Abstract interpretation is a theory in static analysis which verifies a program by using sound approximation of its semantics \cite{CC1977}. Its basic idea is to use an abstract domain to over-approximate the computation on inputs. In \cite{AI2}, this idea has first been developed for verifying DNNs. However, abstract interpretation can be imprecise, due to the non-linearity in DNNs. \cite{deepz} implements a faster Zonotope domain for DNN verification. \cite{deeppoly} puts forward a new abstract domain specially for DNN verification and it is more efficient and precise than Zonotope.
 
The first contribution of this paper is to propose a novel symbolic propagation technique to enhance the precision of abstract interpretation based DNN verification.
For every neuron, we \emph{symbolically} represent, with an expression, how its activation value can be determined by the activation values of  neurons in previous layers. 
By both illustrative examples and experimental results, we show that, comparing with using only abstract domains, our new approach can find significantly tighter constraints over the neurons' activation values. Because abstract interpretation is a sound approximation, with tighter constraints, we may prove properties that cannot be proven by using only abstract domains. For example, we may prove a greater lower bound on the robustness of the DNNs.

Another contribution of this paper is to apply the value bounds derived from our approach on hidden neurons to improve the performance of a state-of-the-art SMT based DNN verifier Reluplex \cite{DBLP:conf/cav/KatzBDJK17}.

Finally, we implement our approach into a software tool and validate it with a few DNNs trained on benchmark datasets such as MNIST, etc. 

\section{Preliminaries}

We recall some basic notions on deep neural networks and abstract interpretation.
For a vector $\bar x \in \mathbb{R}^n$, we use $x_i$ to denote its $i$-th entry.
For a matrix $W \in \mathbb R ^ {m \times n}$, $W_{i,j}$ denotes the entry in its $i$-th row and $j$-th column.
 
\subsection{Deep neural networks}
\begin{figure}[!t]
\begin{center}
\scalebox{0.6}{
\def\layersep{2.5cm}
\def\vertSepFactory{0.7}
\begin{tikzpicture}[shorten >=1pt,->,draw=black!50, node distance=\layersep]
    \foreach \name / \y in {1,...,5}
        \node[
            input neuron, 
            pin=left:
            \ifthenelse{\y=1 \OR \y=2}{
                $\bar x_{\y}$
            }{
                \ifthenelse{\y=5}{
                    $\bar x_m$
                }{
                    $\cdots$
                }
            }
        ] (I-\name) at (0,-\vertSepFactory * \y) {};
        
    \foreach \name / \y in {1,...,9}
        \path[yshift=1.4cm]
            node[hidden neuron] (H1-\name) at (1*\layersep,-\vertSepFactory * \y cm) {};

    \foreach \name / \y in {1,...,9}
        \path[yshift=1.4cm]
            node[hidden neuron] (H2-\name) at (2*\layersep,-\vertSepFactory * \y cm) {};

    \foreach \name / \y in {1,...,9}
        \path[yshift=1.4cm]
            node[hidden neuron] (H3-\name) at (3*\layersep,-\vertSepFactory * \y cm) {};

    \foreach \name / \y in {1,...,9}
        \path[yshift=1.4cm]
            node[hidden neuron] (H4-\name) at (4*\layersep,-\vertSepFactory * \y cm) {};

    \foreach \name / \y in {1,...,5}
        \node[
            output neuron,
            pin={
                [pin edge={->}]
                right:
            \ifthenelse{\y=1 \OR \y=2}{
                $\bar y_{\y}$
            }{
                \ifthenelse{\y=5}{
                    $\bar y_n$
                }{
                    $\cdots$
                }
            }
            }
        ]
        (O-\name) at (5*\layersep, -\vertSepFactory * \y cm) {};

    \foreach \source in {1,...,5}
        \foreach \dest in {1,...,9}
            \path (I-\source) edge (H1-\dest);

    \foreach \source in {1,...,9}
        \foreach \dest in {1,...,9}
            \path (H1-\source) edge (H2-\dest);

    \foreach \source in {1,...,9}
        \foreach \dest in {1,...,9}
            \path (H2-\source) edge (H3-\dest);

    \foreach \source in {1,...,9}
        \foreach \dest in {1,...,9}
            \path (H3-\source) edge (H4-\dest);

    \foreach \source in {1,...,9}
        \foreach \dest in {1,...,5}
            \path (H4-\source) edge (O-\dest);

    \node[annot,above of=H2-1, node distance=1cm, xshift=0.5*\layersep] (hl) {Hidden layer};
    \node[annot,above of=I-1] {Input layer};
    \node[annot,above of=O-1] {Output layer};
\end{tikzpicture}
}
\caption{A fully connected network: Each layer performs the composition of an affine transformation $\mathrm{Affine}(\bar x;W,b)$ and the activated function, where on edges between neurons the coefficients of the matrix $W$ are recorded accordingly.
}
\label{fig:fullyConnectedNetwork}
\end{center}
\end{figure}
We work with deep feedforward neural networks, or DNNs, which can be represented as a function $f:\mathbb{R}^m \to \mathbb{R}^n$, mapping an input $\bar x \in \mathbb{R}^m$ to its corresponding 
output $\bar y =  f(\bar x)   \in \mathbb{R}^n$.
A DNN has in its structure a sequence of layers, including an input layer at the beginning, followed by several hidden layers, and an output layer in the end.
Basically the output of a layer is the input of the next layer.
To unify the representation, we denote the activation values at each layer
as a vector. Thus the transformation between layers can also be seen as
a function in $\mathbb{R}^{m'} \to \mathbb{R}^{n'}$.
The DNN $f$ is the composition of the transformations between layers,
which is typically composed of an affine transformation followed by a non-linear activation function. In this paper we only consider one of the  most commonly used activation function -- the rectified linear unit (ReLU) activation function, defined as 
$$
\mathrm{ReLU}(x)=\max(x,0)
$$
for $x \in \mathbb{R}$ and $\mathrm{ReLU}(\bar x)=(\mathrm{ReLU}(x_1),\ldots,\mathrm{ReLU}(x_n))$ for $\bar x \in \mathbb{R}^n$.

Typically an affine transformation is of the form $\mathrm{Affine}(\bar x;W,b)=W\bar x+b:\mathbb R^m \to \mathbb R^n$, where $W \in \mathbb R ^{n \times m} $ and $b \in \mathbb R^n$. 
Mostly in DNNs we use a \textbf{fully connected layer} to describe the composition of an affine transformation $\mathrm{Affine}(\bar x;W,b)$ and the activation function, if the coefficient matrix $W$ is not sparse and does not have shared parameters. 
We call a DNN with only fully connected layers a fully connected neural network (FNN). Fig.~\ref{fig:fullyConnectedNetwork} gives an intuitive description of fully connected layers and fully connected networks. Apart from fully connected layers, we also have affine transformations whose coefficient matrix is sparse and has many shared parameters, like \textbf{convolutional layers}. Readers can refer to e.g. \cite{AI2} for its formal definition. In our paper, we do not special deal with convolutional layers, because they can be regarded as common affine transformations. In the architecture of DNNs, a convolutional layer is often followed by a non-linear \textbf{max pooling layer}, which takes as an input a three dimensional vector $\bar x \in \mathbb R ^{m \times n \times r}$
with two parameters $p$ and $q$ which divides $m$ and $n$ respectively, defined as 
\begin{align*}
\mathrm{MaxPool_{p,q}}(\bar x)_{i,j,k}=\max\{x_{i',j',k} \mid i' \in(p\cdot (i-1), p \cdot i]\;\wedge\; j' \in (q\cdot (i-1), q\cdot i]\}.
\end{align*}
We call a DNN with only fully connected, convolutional, and max pooling layers a convolutional neural network (CNN).

In the following of the paper, we let the DNN $f$ have $N$ layers, each of which has $m_k$ neurons,  for $0\leq k< N$. Therefore, $m_0=m$ and $m_{N-1}=n$.

\subsection{Abstract interpretation}
Abstract interpretation is a theory in static analysis which verifies a program by using sound approximation of its semantics \cite{CC1977}. Its basic idea is to use an abstract domain to over-approximate the computation on inputs and propagate it through the program.
In the following, we describe its adaptation to work with DNNs.

Generally, on the input layer, we have a concrete domain ${\cal C}$, which includes a set of inputs $X$ as one of its elements. To enable an efficient computation,
we choose an abstract domain ${\cal A}$ to infer the relation of variables in ${\cal C}$. We assume that there is a partial order $\le$ on $\mathcal C$ as well as $\mathcal A$, which in our settings is the subset relation $\subseteq$.
\begin{definition}
A pair of functions $\alpha:{\mathcal C} \to {\mathcal A}$ and $\gamma:{\mathcal A} \to {\mathcal C}$ is a Galois connection, if for any $a \in {\mathcal A}$ and $c \in {\mathcal C}$, we have $\alpha(c) \le a \Leftrightarrow c \le \gamma(a)$.
\end{definition}
Intuitively, a Galois connection $(\alpha,\gamma)$ expresses abstraction and concretization relations between domains, respectively.
Note that, $a \in \mathcal A$ is a sound abstraction of $c \in \mathcal C$ if and only if $c \le \gamma(a)$.

In abstract interpretation, it is important to choose a suitable abstract domain because it determines the efficiency and precision of the abstract interpretation.
In practice, we use a certain type of constraints to represent the abstract elements.
Geometrically, a certain type of constraints correspond to a special shape. E.g., the conjunction of a set of arbitrary linear constraints correspond to a polyhedron. 
Abstract domains that fit for verifying DNN include Intervals, Zonotopes, and Polyhedra, etc.

\begin{itemize}
\item
\textbf{Box.} A box $B$ contains bound constraints in the form of $a \le x_i \le b$. The conjunction of bound constraints express a box in the Euclid space. The form of the constraint for each dimension is an interval, and thus it is also named the Interval abstract domain.
\item
\textbf{Zonotope.} A zonotope $Z$ consists of constraints in the form of $z_i=a_i+\sum_{j=1}^m b_{ij} \epsilon_j$, where $a_i,b_{ij}$ are real constants and $\epsilon_j$ is bounded by a constant interval $[l_j,u_j]$. The conjunction of these constraints express a center-symmetric polyhedra in the Euclid space.
\item
\textbf{Polyhedra.} A Polyhedron $P$ has constraints in the form of linear inequalities, i.e. $\sum_{i=1}^n a_ix_i +  b\le 0$ and it gives a closed convex polyhedron in the Euclid space.
\end{itemize}

The following example shows intuitively how these three abstract domains work:
\begin{example}\label{example:domains}
Let $\bar x \in \mathbb R^2$, and the range of $\bar x$ be $X=\{(1,0),(0,2),(1,2),(2,1)\}$.
With Box, we can abstract the inputs $X$ as $[0,2]\times[0,2]$, and with Zonotope, $X$ can be abstracted as 
$
\left\{x_1=1-\frac 12 \epsilon_1  - \frac 12 \epsilon_3, \enspace x_2=1+ \frac 12 \epsilon_1+\frac 12 \epsilon_2\right\}.
$
where $\epsilon_1,\epsilon_2,\epsilon_3 \in [-1,1]$. 
With Polyhedra, X can be abstracted as
$
\{x_2 \le 2, \enspace x_2 \le -x_1+3, \enspace x_2 \ge x_1-1, \enspace x_2 \ge -2x_1+2\}.
$
Fig.~\ref{fig:aiexample} (left) gives an intuitive description for the three abstractions.
\end{example}

\begin{figure}
\centering
\includegraphics[scale=0.5,trim=30 100 10 100,clip]{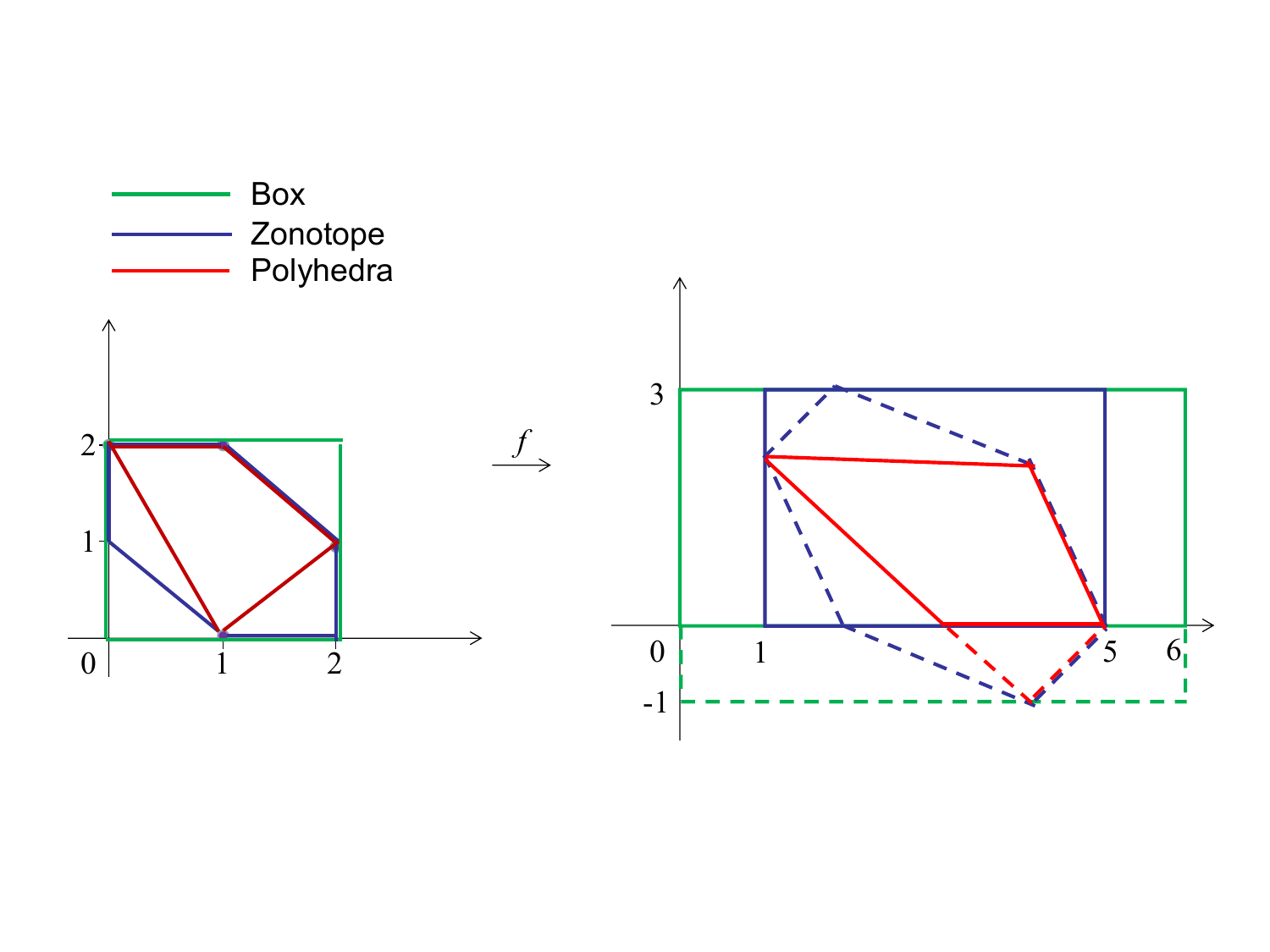} 
\caption{An illustration of Example~\ref{example:domains} and Example~\ref{example:AI}, where on the right the dashed lines gives the abstraction region before the $\mathrm{ReLU}$ operation and the full lines gives the final abstraction $f^\sharp(X^\sharp)$.}
\label{fig:aiexample}
\end{figure}

\section{Symbolic Propagation for Abstract Interpretation based DNN Verification}

In this section, we first describe how to use abstract interpretation to verify DNNs.
Then we present a symbolic propagation method to enhance its precision.

\subsection{Abstract~interpretation~based~DNN~verification}
\label{s:absdnn}

\subsubsection{The DNN verification problem}
The problem of verifying DNNs over a property can be stated formally as follows. 

\begin{definition}
Given a function $f:\mathbb R^m \to \mathbb R^n$ which expresses a DNN, a set of the inputs $X_0 \subseteq \mathbb R^m$, and a property $C \subseteq \mathbb R^n$, verifying
the property is to determine whether $f(X_0) \subseteq C$ holds, where $f(X_0) = \{f(\bar x)~|~\bar x\in X_0\}$.
\end{definition}

Local robustness property can be obtained by letting
\(X_0\) be a robustness region and $C$ be $C_l := \{\bar y \in \mathbb R^n \mid \arg \max_{1 \le i \le n} y_i=l\}$.
where $y$ denotes an output vector and $l$ denotes a label.

A common way of defining robustness region is with norm distance.
We use $\|\bar x-\bar x_0\|_p $ with $p \in [1,\infty]$ to denote the $L_p$ norm distance between two vectors $\bar x$ and $\bar x_0$. In this paper, we use $L_\infty$ norm defined as follows. 
\begin{align*}
\|\bar x\|_\infty=\max_{1 \le i \le n} |x_i|.
\end{align*}
Given an input $\bar x_0 \in \mathbb R^m$ and a perturbation tolerance $\delta>0$, a local robustness region $X_0$ 
can be defined as $B(\bar x_0,\delta):=\{\bar x \mid \|\bar x-\bar x_0\|_p \le \delta\}$.

\subsubsection{Verifying DNNs via abstract interpretation}
Under the framework of abstract interpretation, to conduct verification of DNNs, we first need to choose an abstract domain ${\cal A}$.
Then we represent the set of inputs of a DNN as an abstract element (value) $X^\sharp_0$ in  ${\cal A}$.
After that, we pass it through the DNN layers by applying abstract transformers of the abstract domain. Recall that $N$ is the number of layers in a DNN and ${m_k}$ is the number of nodes in the $k$-th layer. Let $f_k$  (where $1\leq k < N$) be the layer function mapping from ${\mathbb R}^{m_{k-1}}$ to ${\mathbb R}^{m_k}$. We can lift $f_k$ to $T_{f_k}:{\cal P}({\mathbb R}^{m_{k-1}}) \rightarrow {\cal P}({\mathbb R}^{m_k})$ such that $T_{f_k}(X)=\{f_k(\bar x) \mid \bar x \in X\}$.

\begin{definition}
An abstract transformer $T_{f_k}^{\sharp}$ is a function mapping an abstract element $X^\sharp_{k-1}$ in the abstract domain ${\cal A}$  to another abstract element ${X^\sharp}_{k}$. Moreover, $T_{f_k}^{\sharp}$  is sound if $T_{f_k} \circ \gamma \subseteq \gamma \circ T_{f_k}^\sharp$.
\end{definition}

Intuitively, a sound abstract transformer $T_{f_k}^{\sharp}$ maintains a sound relation between the  abstract post-state and the abstract pre-state of a transformer in DNN (such as linear transformation, ReLU operation, etc.).

Let $X_k = f_k(...(f_1(X_0)))$ be the exact set of resulting vectors in $\mathbb{R}^{m_k}$ (i.e., the $k$-th layer) computed over the concrete inputs $X_0$, and $X{^\sharp}_k = T_{f_k}{^\sharp}(...(T_{f_1}{^\sharp}(X{^\sharp}_0)))$ be the corresponding abstract value of the $k$-th layer when using an abstract domain ${\cal A}$. Note that $X_0 \subseteq \gamma({X^\sharp}_{0})$. We have the following conclusion.

\begin{propn}
If $X_{k-1} \subseteq \gamma({X^\sharp}_{k-1})$, then we have $X_{k} \subseteq \gamma({X^\sharp}_{k}) = \gamma \circ T_{f_k}^{\sharp}({{X^\sharp}}_{k-1})$.
\end{propn}

Therefore, when performing abstract interpretation over the transformations in a DNN,
the abstract pre-state ${X^\sharp}_{k-1}$ is transformed into  abstract post-state  ${X^\sharp}_{k}$ by applying the abstract transformer $T_{f_k}^{\sharp}$ which is built on top of an abstract domain. This procedure starts from $k=1$ and continues until reaching the output layer (and getting ${X^\sharp}_{N-1}$).
Finally, we use ${X^\sharp}_{N-1}$ to check the property $C$ as follows:
\begin{equation}\label{equ:ai}
\gamma({X^\sharp}_{N-1})  \subseteq C.
\end{equation}

The following theorem states that this verification procedure based on abstract interpretation is sound for the DNN verification problem.
\begin{theorem}\label{theorem:1}
If Equation (\ref{equ:ai}) holds, then $f(X_0)\subseteq C$.
\end{theorem}

It's not hard to see that the other direction does not necessarily hold due to the potential incompleteness caused by the over-approximation made in both the abstract elements and the abstract transformers $T_{f_k}^{\sharp}$ in an abstract domain.

\begin{example}\label{example:AI} 
Suppose that $\bar x$ takes the value in $X$ given in Example~\ref{example:domains}, and we consider the transformation
$ f(\bar x) =\mathrm{ReLU}\left(\begin{pmatrix}
1 & 2 \\
1 & -1
\end{pmatrix}   \bar x+\begin{pmatrix}
0 \\
1 
\end{pmatrix}\right)$.
Now we use the three abstract domains to calculate the resulting abstraction $f^\sharp(X^\sharp)$ 
\begin{itemize}
\item Box. The abstraction after the affine transformation is $[0,6] \times [-1,3]$,
and thus the final result is $[0,6] \times [0,3]$.
\item Zonotope. After the affine transformation, the zonotope abstraction can be obtained straightforward:
\begin{align*}                                                                           
\left\{y_1=3+\frac 12 \epsilon_1 +  \epsilon_2 -\frac 12 \epsilon_3, \enspace y_2=1- \epsilon_1- \frac 12 \epsilon_2 -\frac 12 \epsilon_3 \mid \epsilon_1,\epsilon_2,\epsilon_3 \in [-1,1] \right\}.
\end{align*}
The first dimension $y_1$ is definitely positive, so it remains the same after the $\mathrm{ReLU}$ operation. The second dimension $y_2$ can be either negative or non-negative, so its abstraction after $\mathrm{ReLU}$ will become a box which only preserves the range in the non-negative part, i.e. $[0,3]$, so the final abstraction is
\begin{align*}                                                                           
\left\{y_1=3+\frac 12 \epsilon_1 +  \epsilon_2 -\frac 12 \epsilon_3, \enspace y_2=\frac 32+ \frac 32 \eta_1 \mid \epsilon_1,\epsilon_2,\epsilon_3,\eta_1 \in [-1,1] \right\},
\end{align*}
whose concretization is $[1,5] \times [0,3]$.
\item Polyhedra. It is easy to obtain the polyhedron before $P_1=\mathrm{ReLU}$ $\{y_2 \le 2,y_2 \ge -y_1+3, y_2 \ge y_1-5,y_2 \le -2y_1+10\}$. Similarly, the first dimension is definitely positive, and the second dimension can be either negative or non-negative, so the resulting abstraction is $(P_1 \wedge (y_2 \ge 0)) \vee (P_1 \wedge (y_2 = 0))$, i.e. $\{y_2 \le 2,y_2 \ge -y_1+3, y_2 \ge 0,y_2 \le -2y_1+10\}$.
\end{itemize}
Fig.~\ref{fig:aiexample} (the right part) gives an illustration for the abstract interpretation with the three abstract domains in this example.
\end{example}
 
The abstract value computed via abstract interpretation can be directly used  to verify properties.
Take the local robustness property, which expresses an invariance on the classification of 
$f$ over a region $B(  \bar{x}_0,\delta)$,
as an example. Let $l_i(\bar x)$ be the confidence of $\bar x$ being labeled as $i$, and $l(\bar x)=\mathrm{argmax}_{i}l_i(\bar x)$ be the label.
It has been shown in \cite{SZSBEGF2014,DBLP:conf/ijcai/RuanHK18} that DNNs are Lipschitz continuous. Therefore, when $\delta$ is small, we have that  $|l_i(\bar x)-l_i(\bar{x}_0)|$ is also small for all labels $i$. That is, if $l_i(\bar{x}_0)$ is significantly greater than $l_j(\bar{x}_0)$ for $j\neq i$, it is highly likely that $l_i(\bar{x})$ is also significantly greater than $l_j(\bar{x})$. It is not hard to see that the more precise the relations among $l_i(\bar{x}_0), l_i(\bar{x}), l_j(\bar{x}_0), l_j(\bar{x})$ computed via abstract interpretation, the more likely we can prove the robustness. Based on this reason, this paper aims to derive techniques to enhance the precision of abstract interpretation such that it can prove some more properties that cannot be proven by the original abstract interpretation.
 
\subsection{Symbolic propagation for DNN verification}

Symbolic propagation can ensure soundness while providing more precise results. In \cite{DBLP:journals/corr/abs-1804-10829}, a technique called symbolic interval propagation is present and we extend it to our abstraction interpretation framework so that it works on all abstract domains.
First, we use the following example to show that using only abstract transformations in an abstract domain may lead to precision loss, while using symbolic propagation could enhance the precision.

\begin{example}\label{example:improvement}
Assume that we have a two-dimensional input $(x_1,x_2) \in [0,1] \times [0,1]$ and a few transformations $y_1 := x_1+x_2$, $y_2 := x_1-x_2$, and $z := y_1+y_2$. Suppose we use the Box abstract domain to analyze the transformations.
\begin{itemize}
\item When using only the Box abstract domain, we have $y_1 \in [0,2]$, $y_2 \in [-1,1]$, and thus $z \in [-1,3]$ (i.e., $[0,2]+[-1,1]$). 
\item By symbolic propagation, we record $y_1=x_1+x_2$ and $y_2=x_1-x_2$ on the neurons $y_1$ and $y_2$ respectively, and then 
get $z=2x_1 \in [0,2]$. This result is more precise than that given by using only the Box abstract domain.
\end{itemize}
\end{example}

Non-relational (e.g., intervals) and weakly-relational abstract domains (e.g., zones, octagons, zonotopes, etc.)\cite{Min17}  may lose precision on the application of the transformations from DNNs. The transformations include affine transformations, ReLU, and max pooling operations.
Moreover, it is often the case for weakly-relational abstract domains that the composition of the optimal abstract transformers of individual statements in a sequence does not result in the optimal abstract transformer for the whole sequence, which has been shown in Example~3 when using only the Box abstract domain. A choice to precisely handle general linear transformations is to use the Polyhedra abstract domain which uses a conjunction of linear constraints as domain representation. However, the Polyhedra domain has the worst-case exponential space and time complexity when handling the ReLU operation (via the join operation in the abstract domain). As a consequence, DNN verification with the Polyhedra domain is impractical for large scale DNNs, which has been also confirmed  in \cite{AI2}.  

In this paper, we leverage symbolic propagation technique to enhance the precision for abstract interpretation based DNN verification. The insight behind is that affine transformations account for a large portion of the transformations in a DNN. Furthermore, when we verify properties such as robustness, the activation of a neuron can often be  deterministic for inputs around an input with small perturbation. Hence, there should be a large number of linear equality relations that can be derived from the composition of a sequence of linear transformations via symbolic propagation. And we can use such linear equality relations to improve the precision of the results given by abstract domains. In Section 6, our experimental results  confirm that, when the perturbation tolerance $\delta$ is small, there is a significant proportion of neurons   whose  ReLU activations are consistent. 

First, given $X_0$, a $\mathrm{ReLU}$ neuron $y:=\mathrm{ReLU}(\sum_{i=1}^n w_i x_i$ $+b)$
can be  classified into one of the following 3 categories (according to its range information): (1) definitely-activated, if the range of $\sum_{i=1}^n w_ix_i+b$ is a subset of $[0,\infty)$, (2) definitely-deactivated, if the range of $\sum_{i=1}^n w_ix_i+b$ is a subset of $(-\infty,0]$, and (3) uncertain, otherwise.

Now we detail our symbolic propagation technique. 
We first introduce a symbolic variable $s_i$ for each node $i$ in the input layer,
to denote the initial value of that node. For a $\mathrm{ReLU}$ neuron $d:=\mathrm{ReLU}(\sum_{i=1}^n w_i c_i+b)$ where $c_i$ is a symbolic variable, we make use of the resulting abstract value of abstract domain at this node to determine whether the value  of this node is definitely greater than 0 or definitely less than 0. If it is a definitely-activated neuron, we record for this neuron  the linear combination $\sum_{i=1}^n w_ic_i+b$ as its symbolic representation (i.e., the value of symbolic propagation). 
If it is a definitely-deactivated neuron, we record for this neuron the value $0$ as its symbolic representation.
Otherwise, we cannot have a linear combination as the symbolic representation and thus a fresh symbolic variable $s_d$ is introduced to denote the output of this $\mathrm{ReLU}$ neuron. We also record the bounds for $s_d$, such that the lower bound for $s_d$ is set to 0 (since the output of a $\mathrm{ReLU}$ neuron is always non-negative) and the upper bound keeps the one obtained by abstract interpretation. 

To formalize the algorithm for ReLU node, we first define the abstract states in the
analysis and three transfer functions for
linear assignments, condition tests and joins respectively.
An abstract state in our analysis
is composed of an abstract element for a numeric domain (e.g., Box) \(\absEle \in \absDom\), a set of free symbolic variables \(\cstVarSet\) (those not equal to any linear expressions), a set of constrained symbolic variables \(\symVarSet\) (those equal to a certain linear expression), and a map from constrained symbolic variables 
to linear expressions
\(\symM::= \symVarSet \Rightarrow \{\sum_{i=1}^n a_i x_i + b \mid x_i \in \cstVarSet \}\). Note that we only allow free variables in the linear expressions in \(\symM\).
In the beginning, all input variables are taken as free symbolic variables. In Algorithm~\ref{algoAssign}, we show the transfer functions for linear assignments \(\absSem{y:=\sum_{i=1}^n w_i x_i + b}\) which over-approximates the behaviors
of \(y:=\sum_{i=1}^n w_i x_i + b\). If \(n>0\) (i.e., the right value expression is not a constant), variable \(y\) is added to the constrained variable set \(\symVarSet\).
All constrained variables in \(\mlexpr = \sum_{i=1}^n w_i x_i + b\) are replaced by their corresponding expressions in \(\symM\), and the map from \(y\) to the new \(\mlexpr\) is recorded in \(\symM\). Abstract numeric element \(\absEle\) is updated
by the transfer function for assignments in the numeric domain \(\absSem{y:=\mlexpr}_{\absDom}\). If \(n\leq 0\), the right-value expression is a
constant, then \(y\) is added to \(\cstVarSet\), and is removed from  \(\symVarSet\)
and \(\symM\). 

The abstract transfer function for condition test is defined as 
\[\absSem{\mlexpr \leq 0}(\absEle,\cstVarSet,\symVarSet,\symM) ::=(
\absSem{\mlexpr \leq 0}_{\absDom}(\absEle),\cstVarSet,\symVarSet,\symM), \]
which only updates the abstract element \(\absEle\) by the transfer function
in the numeric domain \(\absDom\).

\begin{algorithm}
\caption{Transfer function for linear assignments \(\absSem{y:=\sum_{i=1}^n w_i x_i + b}\)}\label{algoAssign}

\KwIn{abstract numeric element $\absEle\in\absDom$, free variables  $\cstVarSet$, constrained variables $\symVarSet$, symbolic map \(\symM\)}

\(\mlexpr\leftarrow \sum_{i=1}^n w_i x_i + b \)

When the right value expression is not a constant

\eIf{\(n>0\)}{
    \For{\(i \in [1,n]\)}{
        \If{\(x_i \in \symVarSet\)}{
            \(\mlexpr = \mlexpr_{|x_i\leftarrow \symM(x_i)}\)
        }
    }
    \(\symM = \symM \cup \{y \mapsto \mlexpr\}\)\quad 
    \(\symVarSet = \symVarSet\cup\{y\}\)\quad
    \(\cstVarSet = \cstVarSet/\{y\}\)\quad
    \(\absEle = \absSem{y:=\mlexpr}_{\absDom}\)
}{
    \(\symM = \symM/(y\mapsto *)\)\quad
    \(\cstVarSet = \cstVarSet\cup\{y\}\)\quad
    \(\symVarSet = \symVarSet/\{y\}\)\quad
    \(\absEle = \absSem{y:=\mlexpr}_{\absDom}\)
}
\KwRet{\((\absEle,\cstVarSet,\symVarSet,\symM)\)}
\end{algorithm}

The join algorithm in our analysis is defined in Algorithm~\ref{algoJoin}. 
Only the constrained variables arising in both input \(\symVarSet_0\) and \(\symVarSet_1\) are
with the same corresponding linear expressions are taken as constrained variables.
The abstract element in the result is obtained by applying the join operator in
the numeric domain \(\sqcup_{\absDom}\).

The transfer function for a ReLU node is defined as
\begin{align*}
    \absSem{y:=\mathrm{ReLU}(\sum_{i=1}^n w_i x_i+b)}(\absEle,\cstVarSet,\symVarSet,\symM)::= 
    \algJoin(\absSem{y>0}(\psi),\absSem{y:=0}(\absSem{y<0})(\psi)),
\end{align*}
where $\psi = \absSem{y:=\sum_{i=1}^n w_i x_i + b}(\absEle,\cstVarSet,\symVarSet,\symM)$.

\begin{algorithm}
\caption{Join algorithm \(\algJoin\)}\label{algoJoin}
\KwIn{\((\absEle_0,\cstVarSet_0,\symVarSet_0,\symM_0)\) and \((\absEle_1,\cstVarSet_1,\symVarSet_1,\symM_1)\)  }

\(\absEle = \absEle_0 \sqcup_{\absDom} \absEle_1\)

\(\symM = \symM_0 \cap \symM_1\)

\(\symVarSet = \{x \mid \exists \mlexpr, x \rightarrow \mlexpr \in \symM\}\)

\(\cstVarSet = \cstVarSet_0 \cup (\symVarSet_0/\symVarSet)\)

\KwRet{\((\absEle,\cstVarSet,\symVarSet,\symM)\)}
\end{algorithm}

For a max pooling node $d:=\max_{1 \le i \le k} c_i$, if there exists some $c_j$ whose lower bound is larger than the upper bound of $c_i$ for all $i \ne j$, 
we set $c_j$ as the symbolic representation for $d$. Otherwise, we introduce a fresh symbolic variable $s_d$ for $d$ and record its bounds wherein its lower (upper) bound is the maximum of the lower (upper) bounds of $c_i$'s. Note that the lower (upper) bound of each $c_i$ can be derived from the abstract value for this neuron given by abstract domain.

The algorithm for max-pooling layer can be defined with the three aforementioned transfer functions as follows:

\[
\begin{array}{c}
    \algJoin(\phi_1,\algJoin(\phi_2,\ldots,\algJoin(\phi_{k-1},\phi_{k}))), \\
    \text{where} \quad \phi_i = \absSem{d:=c_i}\absSem{c_i\geq c_1}\ldots\absSem{c_i\geq c_k}(\absEle,\cstVarSet,\symVarSet,\symM)
\end{array}
\]

\begin{example}
\label{example:symbolicInterval}
For the DNN shown in Figure~\ref{fig:symbolicInterval}(a),
there are two input nodes denoted by symbolic variables $x$ and $y$, two hidden nodes, and one output node. The initial ranges of the input symbolic variables $x$ and $y$ are given, i.e., $[4,6]$ and $[3,4]$ respectively. The weights are labeled on the edges. It is not hard to see that, when using the Interval abstract domain, (the inputs of) the two hidden nodes have bounds $[17,24]$ and $[0,3]$ respectively. For the hidden node with $[17,24]$, we know that this ReLU node is definitely activated, and thus we use symbolic propagation to get a symbolic expression $2x+3y$ to symbolically represent the output value of this node. Similarly, for the hidden node with $[0,3]$, we get a symbolic expression $x-y$. Then for the output node, symbolic propagation results in $x+4y$, which implies that the output range of the whole DNN is $[16,22]$. If we use only the Interval abstract domain without symbolic propagation, we will get the output range $[14,24]$,  which is less precise than $[16,22]$.

For the DNN shown in Figure~\ref{fig:symbolicInterval}(b), we change the initial range of the input variable $y$ to be $[4.5,5]$. For the hidden ReLU node with $[-1,1.5]$, it is neither definitely activated nor definitely deactivated, and thus we introduce a fresh symbolic variable $s$ to denote the output of this node, and set its bound to $[0, 1.5]$. For the output node, symbolic propagation results in $2x+3y-s$, which implies that the output range of the whole DNN is $[20,27]$.
\end{example}

\begin{figure}[t]
\vspace{-10pt}
  \centering
    \includegraphics[scale=0.6,trim=230 330 220 320,clip]{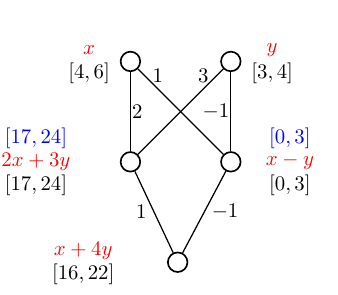}
  \includegraphics[scale=0.6,trim=320 330 110 320,clip]{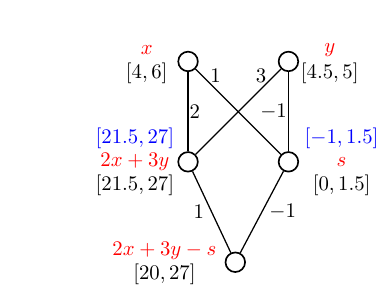}\\
  ~~~~~(a)~~~~~~~~~~~~~~~~~~~~~~~~~~~~~~~~~~~~~~~~~(b)~~~  
  \caption{An illustrative example of symbolic propagation}
  \label{fig:symbolicInterval}
  \vspace{-1mm}
\end{figure}
 
For a definitely-activated neuron, we 
utilize its symbolic representation 
to enhance the precision of abstract domains. We add the linear constraint  $d==\sum_{i=1}^n w_ic_i+b$ into the abstract value at (the input of) this node, via the meet operation (which is used to deal with conditional test in a program) in the abstract domain~\cite{CC1977}. 
If the precision of the abstract value for the current neuron is improved, we may find more definitely-activated neurons 
in the subsequent layers. In other words, 
the analysis based on abstract domain  and our symbolic propagation mutually improves the precision of each other   on-the-fly.

After obtaining  symbolic representation for all the neurons in a layer $k$, the computation proceeds to layer $k+1$. The computation terminates after completing the computation for the output layer. All symbolic representations in the output layer are evaluated to obtain value bounds.

The following theorem shows some results on precision of our symbolic propagation technique.

\begin{theorem}\label{thm:mainprecision}
(1) For an FNN $f: \mathbb R^m \to \mathbb R^n$ and a box region $X \subseteq \mathbb R^m$, the Box abstract domain with symbolic propagation can give a more precise abstraction for $f(X)$ than the Zonotope abstract domain without symbolic propagation.

(2) For an FNN $f: \mathbb R^m \to \mathbb R^n$ and a box region $X \subseteq \mathbb R^m$, the Box abstract domain with symbolic propagation and the Zonotope abstract domain with symbolic propagation gives the same abstraction for $f(X)$.

(3) There exists a CNN $g: \mathbb R^m \to \mathbb R^n$ and a box region $X \subseteq \mathbb R^m$ s.t. the Zonotope abstract domain with symbolic propagation give a more precise abstraction for $g(X)$ than the Box abstract domain with symbolic propagation.
\end{theorem}

\begin{proof}
(1) Since an FNN only contains fully connected layers, we just need to prove that, Box with symbolic propagation (i.e., BoxSymb) is always more precise than Zonotope in the transformations on each RELU neuron $y:=\mathrm{ReLU}(\sum_{i=1}^n w_ix_i$ $+b)$.
Assume that before the transformation, BoxSymb is more precise or as
precise as Zonotope. Since the input is a Box region, the assumption is valid in the beginning. Then we consider three cases: (a)
in BoxSymb, the sign of \(\sum_{i=1}^n w_ix_i+b\) is uncertain,
then it must also be uncertain in Zonotope. In both domains, 
a constant interval with upper bound computed by \(\sum_{i=1}^n w_ix_i+b\) and lower bound as 0 is assigned to $y$ (this can be
inferred from our aforementioned algorithms and \cite{ghorbal2009zonotope}). With our assumption, the upper bound computed by  BoxSymb is more precise than that in Zonotope;
(b) in BoxSymb, the sign of \(\sum_{i=1}^n w_ix_i+b\) is always
positive, then it must be always positive or uncertain in Zonotope.
In the former condition, BoxSymb is more precise because it loses
no precision, while Zonotope can lose precision because of its
limited expressiveness. In the later condition, BoxSymb is more precise
obviously; (c) in BoxSymb, the sign of \(\sum_{i=1}^n w_ix_i+b\) is always negative, then it must be always negative or uncertain
in Zonotope. Similar to case (b), BoxSymb is also more precise in this case. 

(2) Assume that before each transformation on a ReLU neuron $y:=\mathrm{ReLU}(\sum_{i=1}^n w_ix_i$ $+b)$, BoxSymb and
ZonoSymb (Zonotope with symbolic propagation) are with same precision.
This assumption is also valid when the input is a Box region.
Then the evaluation of \(\sum_{i=1}^n w_ix_i+b\) is same in
BoxSymb and ZonoSymb, thus in the three cases:(a) the sign of \(\sum_{i=1}^n w_ix_i+b\) is uncertain, they both compute
a same constant interval for $y$; (b) and (c)
\(\sum_{i=1}^n w_ix_i+b\) is always positive or negative, they both lose no precision.

(3) It is easy to know that, ZonoSymb is more precise or as precise as BoxSymb in all transformations. In CNN, with Max-Pooling layer, we just need to give an example that ZonoSymb can be more precise. Let the Zonotope $X'=\{x_1=2+\epsilon_1+\epsilon_2, \enspace x_2=2+\epsilon_1-\epsilon_2 \mid \epsilon_1,\epsilon_2 \in [-1,1]\}$ and the max pooling node $y=\max\{x_1,x_2\}$. Obviously $X'$ can be obtained through a linear transformation on some box region $X$. With Box with symbolic propagation, the abstraction of $y$ is $[0,4]$, while Zonotope with symbolic propagation gives the abstraction is $[1,4]$. 
\end{proof}

Thm~\ref{thm:mainprecision} gives us some insights: Symbolic propagation technique has a very strong power (even stronger than Zonotope) in dealing with ReLU nodes, while Zonotope gives a more precise abstraction on max pooling nodes. It also provides a useful instruction: When we work with FNNs with the input range being a box, we should use Box with symbolic propagation rather than Zonotope with symbolic propagation since it does not improve the precision but takes more time. Results in Thm~\ref{thm:mainprecision} will also be illustrated in our experiments.

\section{Abstract Interpretation as an Accelerator for SMT-based DNN Verification}
In this section we briefly recall DNN verification based on SMT solvers, and then describe how to utilize the results by abstract interpretation with our symbolic propagation to improve its performance.

\subsection{SMT based DNN verification}
In \cite{DBLP:conf/cav/KatzBDJK17,DBLP:conf/atva/Ehlers17}, two SMT solvers Reluplex and Planet were presented to verify DNNs.
Typically an SMT solver is the combination of
a SAT solver with the specialized decision procedures for other theories. 
The verification of DNNs uses linear arithmetic over real numbers, in which an atom may have the form of $\sum_{i=1}^n w_i x_i \le b$, where $w_i$ and $b$ are real numbers.
Both Reluplex and Planet use the DPLL algorithm to split cases and rule out conflict clauses. They are different in dealing with the intersection. For Reluplex, it inherits rules from the Simplex algorithm and adds a few rules dedicated to $\mathrm{ReLU}$ operation. Through the classical pivot operation, it searches for a solution to the linear constraints, and then apply the rules for $\mathrm{ReLU}$ to ensure the $\mathrm{ReLU}$ relation for every node. Differently, Planet uses linear approximation to over-approximate the DNN, and manage the conditions of $\mathrm{ReLU}$ and max pooling nodes with logic formulas.

\subsection{Abstract interpretation with symbolic propagation as an accelerator}

SMT-based DNN verification approaches are often not efficient, e.g., relying on case splitting for ReLU operation. In the worst case, case splitting is needed for each ReLU operation in a DNN, which leads to an exponential blow-up. In particular, when analyzing large-scale DNNs, SMT-based DNN verification approaches may suffer from the scalability problem and account time out, which is also confirmed experimentally in \cite{AI2}.

In this paper, we utilize the results of abstract interpretation (with symbolic propagation) to accelerate SMT-based DNN verification approaches. 
More specifically, we use the bound information of each ReLU node
(obtained by abstract interpretation) 
to reduce the number of case-splitting, and thus 
accelerate SMT-based DNN verification.
For example, on a neuron $d:=\mathrm{ReLU}(\sum_{i=1}^n w_ic_i+b)$, if we know  that this node is a definitely-activated node according to the bounds given by abstract interpretation, we 
only consider the case $d:=\sum_{i=1}^n w_ic_i+b$ and thus no split is applied.
We remark that, this does not compromise the precision of SMT-based DNN verification while improving their efficiency. 

 
\section{Discussion}
In this section, we discuss the soundness guarantee of our approach. Soundness is an essential 
property of formal verification.

Abstract interpretation is known for its soundness guarantee for analysis and verification \cite{Min17}, since it conducts over-approximation to enclose all the possible behaviors of the original system. Computing over-approximations for a DNN is thus our soundness guarantee in this paper.
As shown in Theorem~\ref{theorem:1}, if the results of abstract interpretation show that the property $C$ holds (i.e., $\gamma({X^\sharp}_N)  \subseteq C$ in Equation~\ref{equ:ai}), then the property also holds for the set of actual executions of the DNN (i.e.,  $f(X_0)\subseteq C$). 
If the results of abstract interpretation can not prove that the property $C$ holds, however, the verification is inconclusive. In this case, the results of the chosen abstract domain are not precise enough to prove the property, and thus more powerful abstract domains are needed.  Moreover, our symbolic propagation also preserves soundness, since it uses symbolic substitution  to compute the composition of linear transformations.  

On the other hand, many existing DNN verification tools
do not guarantee soundness.
For example, Reluplex~\cite{DBLP:conf/cav/KatzBDJK17} (using GLPK), Planet~\cite{DBLP:conf/atva/Ehlers17} (using GLPK), and Sherlock~\cite{DBLP:conf/nfm/DuttaJST18} (using Gurobi)  all rely on the floating point implementation of linear programming solvers, which is unsound. Actually, most state-of-the-art linear programming solvers use floating-point arithmetic and only give approximate solutions which may not be the actual optimum solution or may even lie outside the feasible space~\cite{NS04}. It may happen that a linear programming solver implemented via floating point arithmetic  wrongly claims that a feasible linear system is infeasible or the other way round. In fact, \cite{DBLP:conf/nfm/DuttaJST18}
reports several false positive results in Reluplex, and mentions that this comes from unsound floating point implementation.

\section{Experimental Evaluation}\label{section:experiment}
We present the design and results of our experiments. 

\subsection{Experimental setup}

\textbf{Implementation.} 
c is the first to utilize abstract interpretation to verify DNNs, 
and has implemented all the transformers mentioned in Section~\ref{s:absdnn}. 
Since the implementation of AI$^2$ is not available, we have re-implemented these transformers and refer to them as \air.
We then implemented our symbolic propagation technique based on \air\ and use \air\ as the baseline comparison
in the experiments. Both implementations use general frameworks and thus can run on various abstract domains. 
In this paper, we chose Box (from Apron \footnote{\label{ftn:Apron_Elina_fork}https://github.com/ljlin/Apron\_Elina\_fork}),
T-Zonotope (Zonotope from Apron \footref{ftn:Apron_Elina_fork}) 
and E-Zonotope (Elina Zonotope with the join operator 
\footnote{https://github.com/eth-sri/ELINA/commit/152910bf35ff037671c99ab019c1915e93dde57f}) 
as the underlying domains.

\textbf{Datasets and DNNs.}
We use MNIST  \cite{L1998Gradient} and ACAS Xu \cite{DBLP:conf/emsoft/JeanninGKGSZP15,DBLP:conf/tacas/EssenG14} as the datasets in our experiments. 
MNIST contains $60,000$ $28 \times 28$ grayscale handwritten digits. We can train DNNs to classify the pictures by the written digits on them.
The ACAS Xu system is aimed to avoid airborne collisions and it uses an observation table to make decisions for the aircraft. In \cite{DBLP:journals/corr/abs-1810-04240}, the observation table can be realized by training a DNN instead of storing it.
 
On MNIST, we train seven FNNs and two CNNs. 
The seven FNNs  are of the size $3 \times 20$, $6 \times 20$, $3 \times 50$, $3 \times 100$, 
$6 \times 100$, and $9 \times 200$, where $m \times n$ refers to $m$ hidden layers with $n$ 
neurons in each hidden layer. The CNN1 consists of $2$ convolutional, $1$ max-pooling, $2$
convolutional, $1$ max-pooling, and $3$ fully connected layers in sequence, for a total of 12,412 neurons.
The CNN2 has $4$ convolutional and $3$ fully connected layers (89572 neurons).
On ACAS Xu, we use the same networks as those in \cite{DBLP:conf/cav/KatzBDJK17}. 

\noindent\textbf{Properties.}
We consider the local robustness property with respect to the input region
defined as follows:

\begin{align*}
X_{\bar x, \delta}=\{\bar x' \in \mathbb{R}^m \mid \forall i.1-\delta \le x_i \le x_i' \le 1 \vee x_i=x_i'\}.
\end{align*}
In the experiments, the optional robustness bounds are $0.1,0.2,0.3,0.4,0.5,0.6$.
All the experiments are conducted on an 
openSUSE Leap 15.0 machine with 
Intel
i7
CPU@3.60GHz and 16GB memory.
\subsection{Experimental Results}
%

\begin{table}[p]
    \centering
    \begin{subtable}[h]{1\textwidth}
        \centering
        \scalebox{1}{
            \begin{tabular}{c|cc|ccc|c}
            \toprule
             & \multicolumn{2}{c|}{\air} & \multicolumn{3}{|c|}{Symb} & \multirow{2}{*}{Planet}\\
            \cmidrule(r){2-3} \cmidrule(l){4-6}  
            {}  & TZono & EZono & Box & TZono & EZono &  \\
            \midrule
            FNN1  & 28.23348\% & 28.02098\% &   9.69327\% &    9.69327\% & 9.69327\% &    7.05553\% \\
            FNN2  & 24.16382\% & 22.13319\% &   1.76704\% &    1.76704\% & 1.76704\% &    0.89089\% \\
            FNN3  & 26.66453\% & 26.30852\% &   6.88656\% &    6.88656\% & 6.88656\% &    4.51223\% \\
            FNN4  & 28.47243\% & 28.33535\% &   5.13645\% &    5.13645\% & 5.13645\% &    2.71537\% \\
            FNN5  & 35.61163\% & 35.27187\% &   3.34578\% &    3.34578\% & 3.34578\% &    0.14836\% \\
            FNN6  & 38.71020\% & 38.57376\% &   7.12480\% &    7.12480\% & 7.12480\% &    1.94230\% \\
            FNN7  & 41.76517\% & 41.59382\% &   5.52267\% &    5.52267\% & 5.52267\% &    \quad 1h TIMEOUT  \\
            CNN1  & 24.19607\% & 24.13725\% &   21.78279\% &   7.58917\% & 7.56223\% &    \quad 8h TIMEOUT  \\
            CNN2  &OOM         & OOM        &   1.09146\%  &   OOM       & OOM       &    \quad 8h TIMEOUT  \\
            \bottomrule
            \end{tabular}
        }
        \caption{
            Bound proportions (smaller is better) of different abstract interpretation approaches with the robustness bound $\delta \in \{0.1, 0.2, 0.3, 0.4,0.5, 0.6\}$, and the fixed
            input pictures 767, 1955, and 2090;
        }
        \label{tab:3picboundrate}
    \end{subtable} 
    \begin{subtable}[h]{1\textwidth}
        \centering
        \scalebox{0.805}{
            \begin{tabular}{c|cccccc|cccccc|cc}
            \toprule
             & \multicolumn{6}{c|}{\air} & \multicolumn{6}{|c|}{Symb} & \multicolumn{2}{|c}{\multirow{2}{*}{Planet}}\\
            \cmidrule(r){2-7} \cmidrule(l){8-13}  
            {} & \multicolumn{2}{c}{Box}& \multicolumn{2}{c}{TZono} &\multicolumn{2}{c|}{EZono} & \multicolumn{2}{c}{Box} & \multicolumn{2}{c}{TZono} &\multicolumn{2}{c|}{EZono} &  \\
            \hline
            FNN1 & 11.168    & 0.2 &13.482     & 0.5 &44.05      & 0.5 &12.935     & 0.6  &17.144     & 0.6 &45.88     & 0.6  &20.179    & 0.6 \\
            FNN2 & 12.559    & 0   &16.636     & 0.2 &50.59      & 0.2 &15.075     & 0.5  &22.333     & 0.5 &49.92     & 0.5  &35.84     & 0.6  \\
            FNN3 & 12.699    & 0.2 &18.748     & 0.3 &49.812     & 0.3 &19.042     & 0.6  &28.128     & 0.6 &54.77     & 0.6  &76.106    & 0.6 \\
            FNN4 & 15.583    & 0.1 &29.495     & 0.3 &58.892     & 0.3 &37.716     & 0.6  &56.47      & 0.6 &76.00     & 0.6  &351.139   & 0.6  \\
            FNN5 & 28.963    & 0   &81.49      & 0.2 &149.791    & 0.2 &90.268     & 0.4  &154.222    & 0.4 &173.263   & 0.4  &1297.485  & 0.6  \\
            FNN6 & 62.766    & 0   &398.565    & 0.1 &538.076    & 0.1 &323.328    & 0.3  &650.629    & 0.3 &745.454   & 0.3  &15823.208 & 0.3 \\
            FNN7 & 111.955   & 0   &1674.465   & 0   &1627.72    & 0   &642.978    & 0.3  &1524.975   & 0.3 &1489.604  & 0.3  &\quad 1h TIMEOUT  \\
            CNN1 & 2340.828  & 0   &6717.57    & 0.2 &94504.195  & 0.2 &5124.681   & 0.2  &8584.555   & 0.3 &45452.102 & 0.3  &\quad 8h TIMEOUT  \\
            CNN2 & 41292.291 & 0   &OOM        & 0   &OOM        & 0   &105850.271 & 0.3  &OOM        & 0   &OOM       & 0    &\quad 8h TIMEOUT  \\
            \bottomrule
            \end{tabular}
        }

        \caption{
            The time (in second) and the maximum robustness bound $\delta$ which can be verified 
            through the abstract interpretation technique and the planet bound, with optional 
            $\delta \in \{0.1,0.2,0.3,0.4,0.5,0.6\}$ and the fixed input picture 2090;
        }
    \end{subtable} 
    \begin{subtable}{1\textwidth}
        \centering
        \scalebox{0.57}{
             \begin{tabular}{c|c|c|c|c|c|c|c}
            \toprule
             & \multicolumn{3}{c|}{\air} & \multicolumn{3}{|c|}{Symb} & \multirow{2}{*}{Planet}\\
            \cmidrule(r){2-4} \cmidrule(l){5-7}  
            {}  & Box & TZono & EZono & Box & TZono & EZono &  \\
            \midrule
            FNN1(60)   &57    \enspace 44    \enspace 34     & 59    \enspace 52    \enspace 38     &59     \enspace 52    \enspace 38     &59    \enspace 53    \enspace 44    &59\enspace53\enspace44          &59\enspace53\enspace44            &59\enspace56\enspace55 \\
            FNN2(120)  &103   \enspace 59    \enspace 38     & 118   \enspace 109   \enspace 66     &118    \enspace 111   \enspace 66     &118   \enspace 113   \enspace 107   &118\enspace113\enspace107       &118\enspace113\enspace107         &119\enspace114\enspace110  \\
            FNN3(150)  &136   \enspace 93    \enspace 66     & 141   \enspace 127   \enspace 85     &141    \enspace 127   \enspace 85     &143   \enspace 133   \enspace 110   &143\enspace133\enspace110       &143\enspace133\enspace110         &146\enspace142\enspace135 \\
            FNN4(300)  &250   \enspace 144   \enspace 105    & 294   \enspace 209   \enspace 130    &294    \enspace 209   \enspace 130    &295   \enspace 254   \enspace 182   &295\enspace254\enspace182       &295\enspace254\enspace182         &296\enspace276\enspace254  \\
            FNN5(600)  &289   \enspace 160   \enspace 106    & 513   \enspace 200   \enspace 125    &513    \enspace 200   \enspace 125    &589   \enspace 510   \enspace 236   &589\enspace510\enspace236       &589\enspace510\enspace236         &593\enspace558\enspace493  \\
            FNN6(1200) &472   \enspace 247   \enspace 181    & 782   \enspace 339   \enspace 195    &782    \enspace 339   \enspace 195    &1176  \enspace 790   \enspace 250   &1176\enspace790\enspace250      &1176\enspace790\enspace250        &1189\enspace1089\enspace772  \\
            FNN7(1800) &469   \enspace 271   \enspace 177    & 770   \enspace 350   \enspace 200    &775    \enspace 350   \enspace 200    &1773  \enspace 741   \enspace 263   &1773\enspace741\enspace263      &1773\enspace741\enspace263        &1h TIMEOUT  \\
            CNN1(12412) &12226 \enspace 11788 \enspace 11280 & 12371 \enspace 12119 \enspace 11786  &12371  \enspace 12122 \enspace 11786  &12373 \enspace 12094 \enspace 11659 &12376\enspace12193\enspace11877 &12376\enspace12196 \enspace11877  &8h TIMEOUT  \\
            CNN2(89572) &85793 \enspace 77241 \enspace 70212 & OOM                                  &OOM                                   &89190 \enspace 86910 \enspace 81442 &OOM                             &OOM                               &8h TIMEOUT  \\
            \bottomrule
            \end{tabular}
        }
        \caption{
            The number of hidden $\mathrm{ReLU}$ neurons whose behavior can be decided with the 
            bounds our abstract interpretation technique and Planet provide, with optional robustness
            bound $\delta \in \{0.1,0.4,0.6\}$ and the fixed input picture 767.
        }
    \end{subtable} 
    \caption{Experimental results of abstract interpretation for MNIST DNNs with different approaches}
    \label{fig:experimets}
\end{table}
We compare seven approaches: \air\  with Box, T-Zonotope and E-zonotope as underlying domains and  
Symb (i.e., our enhanced abstract  interpretation with symbolic propagation) with Box, 
T-Zonotope and E-zonotope as underlying domains, and Planet~\cite{DBLP:conf/atva/Ehlers17}, which
serves as the benchmark verification approach 
(for its ability to compute bounds).  
\paragraph{Improvement on Bounds}
To see the improvement on bounds, we compare the 
output ranges of the above seven approaches on different inputs
$\bar x$ and
different tolerances $\delta$.
Table~\ref{fig:experimets}(a) reports the results on three inputs $\bar x$ 
(No.767, No.1955 and No.2090 in the MNIST training dataset) 
and six tolerances $\delta \in \{0.1,0.2,0.3,0.4,0.5,0.6\}$. 
In all our experiments, we set TIMEOUT as one hour for each FNN
and eight hours for each CNN for 
a single run with an input, and a tolerance $\delta$. 
In the table, TZono and EZono are shorts for T-Zonotope and E-Zonotope.

For each running we get a gap with an upper and lower bound for each neuron. 
Here we define the \emph{the bound proportion} to statistically describe how precise the range an approach gives. 
Basically given an approach (like Symb with Box domain), the  bound proportion of this approach is the average of 
the ratio of the gap length of the neurons on the output layer and that obtained using \air\ with Box. 
Naturally \air\  with Box always has the bound proportion $1$, and the smaller the  bound proportion is, 
the more precise the ranges the approach gives are.

In Table \ref{fig:experimets}(a), every entry is the average bound proportion 
over three different inputs and six different tolerances.
 OOM stands for out-of-memory, 1h TIMEOUT for the one-hour timeout, and 8h TIMEOUT for the eight-hour timeout. 
 We can see that, in general, 
Symb with Box, T-Zonotope and E-zonotope
can achieve much better bounds than \air\  with Box, T-Zonotope and E-zonotope do.
These bounds are closer to what Planet gives, except for FNN5 and FNN6.
E-zonotope is slightly more precise than T-Zonotope.
On the  other hand, while Symb can 
return in a reasonable time in most cases, 
Planet cannot terminate in an hour (resp. eight hours) for FNN7  (resp. CNN1 and CNN2), 
which have $1,800$, $12,412$ and $89,572$ hidden neurons, respectively. 
Also we can see that results in Theorem~\ref{thm:mainprecision} are illustrated here. 
More specifically, 
(1) Symb with Box domain is more precise than \air\  with T-Zonotope and E-Zonotope on FNNs;
(2) Symb with Box, T-Zonotope and E-Zonotope are with the same precision on FNNs;
(3) Symb with T-Zonotope and E-Zonotope are more precise than Symb with Box on
CNNs.

\begin{figure}[!t]
        \begin{subfigure}[b]{0.5\textwidth}
                \centering
                \includegraphics[width=.85\linewidth]{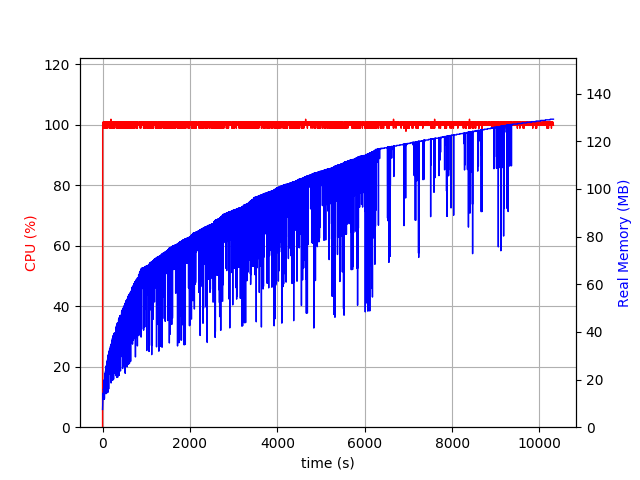}
                \caption{Box}
        \end{subfigure}%
        \begin{subfigure}[b]{0.5\textwidth}
                \centering
                \includegraphics[width=.85\linewidth]{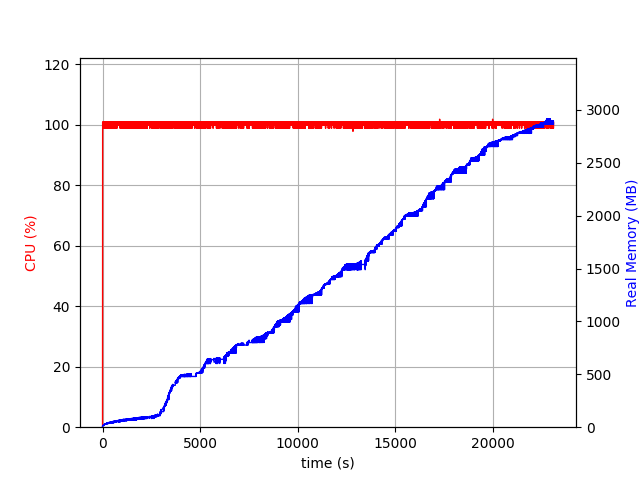}
                \caption{SymBox}
        \end{subfigure}%

        \begin{subfigure}[b]{0.5\textwidth}
                \centering
                \includegraphics[width=.85\linewidth]{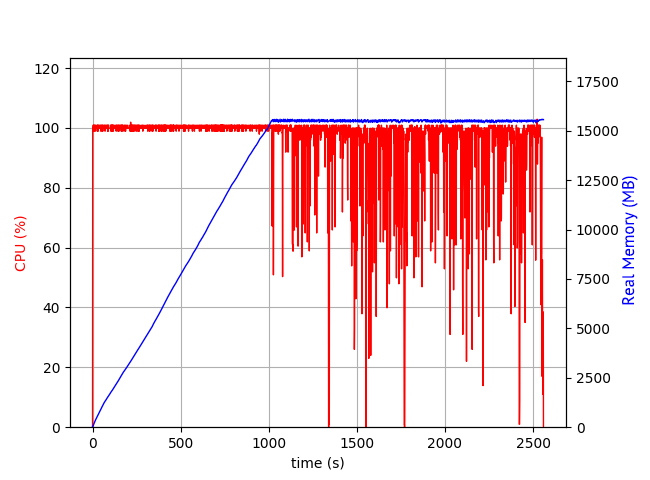}
                \caption{TZono}
        \end{subfigure}%
        \begin{subfigure}[b]{0.5\textwidth}
                \centering
                \includegraphics[width=.85\linewidth]{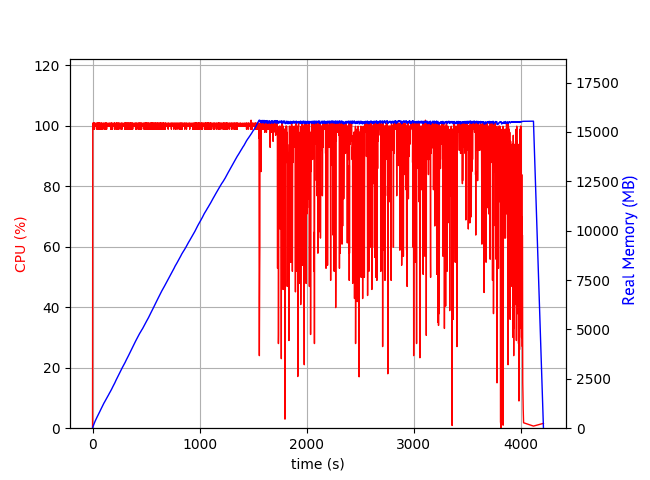}
                \caption{SymTZono}
        \end{subfigure}

        \begin{subfigure}[b]{0.5\textwidth}
                \centering
                \includegraphics[width=.85\linewidth]{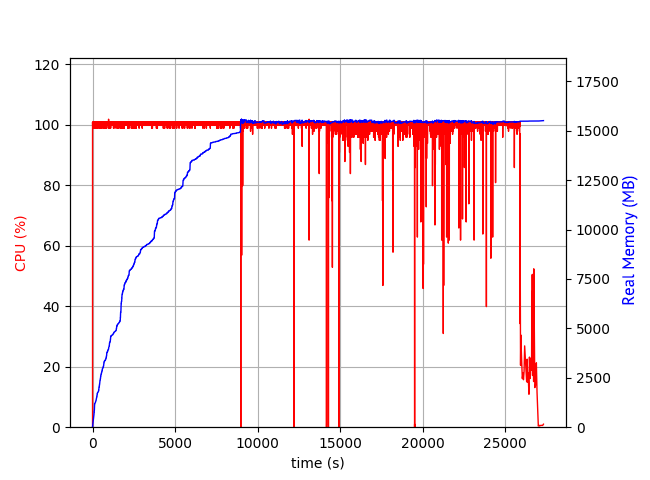}
                \caption{EZono}
        \end{subfigure}%
        \begin{subfigure}[b]{0.5\textwidth}
                \centering
                \includegraphics[width=.85\linewidth]{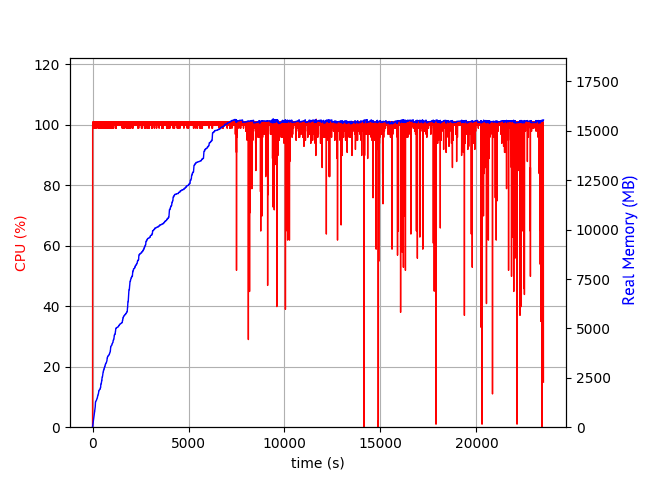}
                \caption{SymEZono}
        \end{subfigure}
        \caption{CPU and memory usage}
        \label{fig:cpumem}
\end{figure}

According to memory footprint, both \air\ and Symb with T-Zonotope or
E-Zonotope need more memory 
than them with Box do, and will crash on large networks, such as CNN2, because of running out of memory.
Figure~\ref{fig:cpumem} shows how CPU and resident memory usage change over time.
The horizontal axis in the figure is the time, in seconds, the vertical axis 
corresponding to the red line is the CPU usage percentage, and the vertical 
axis corresponding to the blue line is the memory usage, in MB.


\paragraph{Greater Verifiable Robustness Bounds}

Table \ref{fig:experimets}(b) shows the results of using the obtained bounds to help verify the robustness property. 
We consider a few thresholds for robustness tolerance, i.e., $\{0.1,0.2,0.3,0.4,0.5,0.6\}$, and find that Symb can verify many more 
cases than \air\ do with comparable time consumption (less than 2x in most cases, and sometimes even faster).

\paragraph{Proportion of Activated/Deactivated ReLU Nodes}

Table \ref{fig:experimets}(c) reports the number of hidden neurons whose ReLU behaviour (i.e., activated or deactivated) has been consistent within the tolerance $\delta$. 
Compared to \air, our Symb can decide the ReLU behaviour with a much higher percentage.  
%
\begin{table}[h!]
    \centering
    \begin{tabular}{c|cc|ccc|c}
        \toprule
        \multirow{2}{*}{$\delta$} & \multicolumn{2}{c|}{\air} & \multicolumn{3}{|c|}{Symb} & \multirow{2}{*}{Planet}\\
        \cmidrule(r){2-3} \cmidrule(l){4-6}  
        {}  & TZono & EZono & Box & TZono & EZono &  \\
        \hline
         0.1  & \quad 7.13046\%   & 7.08137\%  & \quad 6.15622\%   & \quad 6.15622\%  & 6.15622\% & \quad 5.84974\%  \\
         0.2  & \quad 11.09230\%  & 10.88775\% & \quad 6.92011\%   & \quad 6.92011\%  & 6.92011\% & \quad 6.11095\%  \\
         0.3  & \quad 18.75853\%  & 18.32059\% & \quad 8.21241\%   & \quad 8.21241\%  & 8.21241\% & \quad 6.50692\%  \\
         0.4  & \quad 30.11872\%  & 29.27580\% & \quad 10.31225\%  & \quad 10.31225\% & 10.31225\% & \quad 7.04413\%  \\
         0.5  & \quad 45.13963\%  & 44.25026\% & \quad 14.49276\%  & \quad 14.49276\% & 14.49276\% & \quad 7.96402\%  \\
         0.6  & \quad 55.67772\%  & 54.88288\% & \quad 20.03251\%  & \quad 20.03251\% & 20.03251\% & \quad 9.02688\%  \\
        \bottomrule
    \end{tabular}
    \vspace{3em}
    \caption{
        Bound proportions (smaller is better) for 1000 randomly sampled pictures from MNIST testing set on FNN1 with $\delta \in \{0.1,0.2,0.3,0.4,0.5,0.6\}$.
    }
    \label{tab:1000pics}
\end{table} 

We remark that, although the experimental results presented above are based 
on $3$ fixed inputs, more extensive experiments have already been conducted to 
confirm that the conclusions are general. We randomly sample 1000 pictures 
(100 pictures per label) from the MNIST dataset, and compute the bound proportion 
for each of the pair $(m,\delta)$ 
where $m$ refers to the seven approaches in Table \ref{fig:experimets}
and $\delta \in \{0.1, 0.2, 0.3, 0.4, 0.5, 0.6\}$ on FNN1. 
Each entry corresponding to $(m,\delta)$ in Table~(\ref{tab:1000pics}) is the average of bound proportions of approach $m$ over 1000 pictures 
and fixed tolerance $\delta$.
Then we get the average of the bound proportion of 
\air\ with TZono/EZono,
Symb with Box/TZono/EZono,
 and Planet over six different tolerances are
$27.98623\%$,$ 27.44977\%$,$11.02104\%$, $11.02104\%$, $11.02104\%$, $7.08377\%$, respectively, which are very 
close to the first row of Table~\ref{fig:experimets}(a).

\paragraph{Comparison with the bounded powerset domain}

In AI\(^2\)~\cite{AI2}, the bounded powerset domains are used
to improve the precision. In \air , we also implemented such bounded powerset domains instantiated
by Box, T-Zonotope and E-Zonotope  domains, with 32 as the 
bound of the number of the abstract elements in a disjunction. 
The comparison of the performance on the powerset domains
with our symbolic propagation technique (with underlying domains rather than powerset domains) is shown in Table~\ref{tab:3picboundrate}.
We can see that our technique is much more precise than the powerset domains.
The time and memory consumptions of the powerset domains are both around 32 times as much as the underlying domains, which are more than those of our technique.
\begin{table}
    \centering
    \scalebox{0.95}{
        \begin{tabular}{c|ccc|ccc|c}
        \toprule
         & \multicolumn{3}{c|}{\air} & \multicolumn{3}{|c|}{Symb} & \multirow{2}{*}{Planet}\\
        \cmidrule(r){2-4} \cmidrule(l){5-7}  
        {}  & Box32 & TZono32 & EZono32 & Box & TZono & EZono &  \\
        \midrule
        FNN1 &89.65790\% &  20.68675\% &  15.87726\%       &   9.69327\% &    9.69327\% & 9.69327\% &    7.05553\% \\
        FNN2 &89.42070\% &  16.27651\% &   8.18317\%       &   1.76704\% &    1.76704\% & 1.76704\% &    0.89089\% \\
        FNN3 &89.43396\% &  21.98109\% &  12.42840\%       &   6.88656\% &    6.88656\% & 6.88656\% &    4.51223\% \\
        FNN4 &89.44806\% &  25.97855\% &  13.05969\%       &   5.13645\% &    5.13645\% & 5.13645\% &    2.71537\% \\
        FNN5 &89.16034\% &  29.61022\% &  17.88676\%       &   3.34578\% &    3.34578\% & 3.34578\% &    0.14836\% \\
        FNN6 &89.30790\% &   OOM       &  22.60030\%       &   7.12480\% &    7.12480\% & 7.12480\% &    1.94230\% \\
        FNN7 &88.62267\% &   OOM       &  \quad 1h TIMEOUT &   5.52267\% &    5.52267\% & 5.52267\% &    \quad 1h TIMEOUT  \\
        \bottomrule
        \end{tabular}
    }
    \vspace{1em}
    \caption{
        Bound proportions (smaller is better) of different abstract interpretation approaches with the robustness bound $\delta \in \{0.1, 0.2, 0.3, 0.4,0.5, 0.6\}$, and the fixed
        input pictures 767, 1955, and 2090. 
            Note that each entry gives the average bound proportion over six different tolerance and three pictures.
    }
    \label{tab:3picboundrate}
\end{table} 
\paragraph{Faster Verification}
\begin{table}[h]
\centering
\scalebox{0.82}{
\begin{tabular}[htp]{clc|clrclrclrclrclrc|cr}
  \toprule
  &&&&
  \multicolumn{2}{c}{$\delta = 0.1$} 
  &&  
  \multicolumn{2}{c}{$\delta = 0.075$} 
  &&  
  \multicolumn{2}{c}{$\delta = 0.05$} 
  &&  
  \multicolumn{2}{c}{$\delta = 0.025$} 
  &&  
  \multicolumn{2}{c}{$\delta = 0.01$} 
  &&& 
  \multicolumn{1}{c}{Total}
  \\
  \cline{5-6} \cline{8-9} \cline{11-12} \cline{14-15} \cline{17-18}
  &&&&   
      Result & Time
  &&  
      Result & Time
  &&  
      Result & Time
  &&  
      Result & Time
  &&  
      Result & Time
  &&&  
  \multicolumn{1}{c}{Time}
  \\
  \midrule
  \multirow{2}{*}{Point 1}
  &
  Reluplex &&& 
  \sat{} & 39 
  && 
  \sat{} & 123
  && 
  \sat{} & 14
  && \unsat{} & 638
  && \unsat{} & 64
  &&& 879
  \\
  & 
  Reluplex + ABS &&& 
  \sat{} & 45 
  && 
  \sat{} & 36
  && 
  \sat{} & 14
  && \unsat{} & 237
  && \unsat{} & 36
  &&& 368
  \\
  \midrule  
  \multirow{2}{*}{Point 2} &
  Reluplex &&& 
  \unsat{} & 6513
  && 
  \unsat{} & 1559
  && 
  \unsat{} & 319
  && \unsat{} & 49
  && \unsat{} & 11
  &&& 8451
  \\
  &
  Reluplex + ABS&&& 
  \unsat{} & 141
  && 
  \unsat{} & 156
  && 
  \unsat{} & 75
  && \unsat{} & 40
  && \unsat{} & 0
  &&& 412
  \\
  \midrule
  \multirow{2}{*}{Point 3}
  &
  Reluplex &&& 
  \unsat{} & 1013
  && 
  \unsat{} & 422
  && 
  \unsat{} & 95
  && \unsat{} & 79
  && \unsat{} & 6
  &&& 1615
  \\
  &
  Reluplex + ABS&&& 
  \unsat{} & 44
  && 
  \unsat{} & 71
  && 
  \unsat{} & 0
  && \unsat{} & 0
  && \unsat{} & 0
  &&& 115
  \\
  \midrule
  \multirow{2}{*}{Point 4}
  &
  Reluplex &&& 
  \sat{} & 3   
  && 
  \sat{} & 5
  && 
  \sat{} & 1236 
  && \unsat{} & 579
  && \unsat{} & 8 
  &&& 1831
  \\
  &
  Reluplex + ABS&&& 
  \sat{} & 3   
  && 
  \sat{} & 7
  && 
  \unsat{} & 442 
  && \unsat{} & 31
  && \unsat{} & 0 
  &&& 483
  \\
  \midrule
  \multirow{2}{*}{Point 5}&
  Reluplex &&& 
  \unsat{} & 14301   
  && 
  \unsat{} & 4248
  && 
  \unsat{} & 1392
  && \unsat{} & 269
  && \unsat{} & 6
  &&& 20216
  \\
  &
  Reluplex + ABS &&& 
  \unsat{} & 2002   
  && 
  \unsat{} & 1402
  && 
  \unsat{} & 231
  && \unsat{} & 63
  && \unsat{} & 0
  &&& 3698
  \\
  \bottomrule
\end{tabular}%
}
\vspace{3mm}
\caption{The satisfiability on given $\delta$, and the time (in second) with and without bounds generated by abstract interpretation with symbolic propagation on the Box domain.}
\label{table:experiment4}
\end{table}%
In this part we use the networks of ACAS Xu. To evaluate the benefits of tighter bounds for SMT-based tools, we give the bounds obtained by abstract interpretation (on Box domain with symbolic propagation)  to Reluplex~\cite{DBLP:conf/cav/KatzBDJK17} and observe the performance difference.
The results are shown in Table~\ref{table:experiment4}. Each cell shows the 
satisfiability (i.e., SAT if an adversarial example is found) and the running time without
or with given bounds. The experiments are conducted on different  
$\delta$ values
(as in~\cite{DBLP:conf/cav/KatzBDJK17}) 
and a fixed network (nnet1\_1~\cite{DBLP:conf/cav/KatzBDJK17})
 and 5 fixed points (Point 1 to 5 in \cite{DBLP:conf/cav/KatzBDJK17}). 
 The running time our technique spends on deriving the bounds are all less than $1$ second.
 Table~\ref{table:experiment4} shows that tighter initial
 bounds bring significant benefits to Reluplex with an overall 
 $({\frac{1}{5076}-\frac{1}{32992})}/{\frac{1}{32992}} = 549.43\%$
 speedup (9.16 hours compared to 1.41 hours).
 However, it should be noted that, on one specific case (i.e., \(\delta=0.1\) at Point 1 and \(\delta=0.075\) at point 4), 
 the tighter initial bounds slow Reluplex, which means that the speedup is
 not guaranteed on all cases.
 For the case \(\delta=0.05\) at point 4, Reluplex gives SAT and Reluplex+ABS gives UNSAT. This may result from a floating point arithmetic error.


\section{Related Work}

Verification of neural networks can be traced back to  \cite{DBLP:conf/cav/PulinaT10}, where the network is encoded after approximating every sigmoid activation function with a set of piecewise linear constraints and then solved with an SMT solver. It works with a network of 6 hidden nodes.
More recently, by considering DNNs with ReLU activation functions, the verification approaches include constraint-solving \cite{DBLP:conf/cav/KatzBDJK17,LM2017,DBLP:conf/atva/Ehlers17,NKPSW2017}, layer-by-layer exhaustive search \cite{DBLP:conf/cav/HuangKWW17}, global optimisation \cite{DBLP:conf/ijcai/RuanHK18,SherLock,RWSHKK2019}, abstract interpretation \cite{AI2,deepz,deeppoly}, functional approximation \cite{XTJ2018}, and reduction to two-player game \cite{WHK2017,WWRHK2019}, etc.
%
More specifically, \cite{DBLP:conf/cav/KatzBDJK17} presents an SMT solver Reluplex to verify
properties
on DNNs with fully-connected layers.
\cite{DBLP:conf/atva/Ehlers17}
presents another SMT solver Planet
which combines linear approximation and interval arithmetic to work with fully connected and max pooling layers.
Methods based on SMT solvers 
do not scale well,
e.g., Reluplex can only work with DNNs with a few hidden neurons.

The above works are mainly for the verification of local robustness.
Research has been conducted to compute other properties, e.g., 
the output reachability. An exact computation of output reachability can be utilised to verify local robustness. 
In \cite{DBLP:conf/nfm/DuttaJST18}, Sherlock, an algorithm based on local and global search and mixed integer linear programming (MILP), is put forward to calculate the output range of a given label when the inputs are restricted to a small subspace.
\cite{DBLP:conf/ijcai/RuanHK18} presents another algorithm for output range analysis, and their algorithm is workable for all Lipschitz continuous DNNs, including all layers and activation functions mentioned above.
In \cite{DBLP:journals/corr/abs-1804-10829}, the authors use symbolic interval propagation to calculate output range. Compared with \cite{DBLP:journals/corr/abs-1804-10829}, our approach  fits for general abstract domains, while their symbolic interval propagation  is designed specifically for symbolic intervals. 
 
AI\(^2\)\cite{AI2} is the first to use
abstract interpretation to verify DNNs.
They define a class of functions called conditional affine transformations (CAT) to characterize DNNs containing fully connected, convolutional and max pooling layers with the ReLU activation function.
They use Interval and Zonotope as the abstract domains and the powerset technique on Zonotope.
Compared with AI${}^2$, we use symbolic propagation rather than powerset extension techniques to enhance the precision of abstract interpretation based DNN verification. Symbolic propagation is more lightweight than powerset extension. Moreover, we also use the bounds information given by  abstract interpretation to accelerate SMT based DNN verification. 
DeepZ~\cite{deepz} and DeepPoly~\cite{deeppoly} propose two
specific abstract domains tailored to DNN verification, in order to improve the precision of
abstract interpretation on the verification on DNNs. In contrast, our
work is a general approach that can be applied on various domains.

\section{Conclusion}

In this paper, we have explored more potential of  abstract interpretation
on the verification over DNNs. We have proposed to use symbolic propagation on 
abstract interpretation to take advantage of the linearity in most part of the DNNs,
which achieved significant improvements in terms of the precision and memory usage. 
This is based on a key observation that, for local robustness verification of DNNs 
where a small region of the input space is concerned, a considerable percentage of hidden neurons 
remain active or inactive for all possible inputs in the region. 
For these neurons, their ReLU activation function can be replaced by a linear function. 
Our symbolic propagation iteratively computes for each neuron this information and utilize 
the computed information to improve the performance. 

This paper has presented with formal proofs three somewhat surprising 
theoretical results, which are then affirmed by  our experiments. 
These results have enhanced our theoretical and practical understanding 
about the abstract interpretation based DNN verification and symbolic propagation. 

This paper has also applied the tighter bounds of variables on hidden neurons from our approach to 
improve the performance of the state-of-the-art SMT based DNN verification tools, 
like Reluplex. The speed-up rate is up to 549\% in our experiments. 
We believe this result sheds some light on the potential in improving the scalability 
of SMT-based DNN verification: In addition to improving the performance through
enhancing the SMT solver for DNNs, an arguably easier way is to take an abstract interpretation technique 
(or other techniques that can refine the constraints) as a pre-processing.

\section*{Acknowledgements}
This work is supported by the Guangdong Science and Technology Department (no. 2018B010107004) and the NSFC Program (No. 61872445).
We also thank anonymous reviewers for detailed comments.

\bibliographystyle{splncs04}
\bibliography{main}

\end{document}